\documentclass{llncs}
\usepackage{times}
\usepackage{helvet}
\usepackage{courier}
\usepackage{amssymb}
\usepackage{amsmath}
\usepackage{ntheorem}
\usepackage{graphics}
\newcommand{\MICE}{{\sc Mice}}
\newcommand{\GUROBI}{{Gurobi}}
\newcommand{\CPLEX}{{IBM CPLEX}}

\newcommand{\DEVIATION}{\texttt{Deviation}}
\newcommand{\GCC}{\texttt{Gcc}}
\newcommand{\ELEMENT}{\texttt{Element}}
\newcommand{\ALLDIFF}{\texttt{Alldifferent}}

\newcommand{\INTERDISTANCE}{\texttt{InterDistance}}

\newcommand{\DISTANCE}{\texttt{DistanceXYZ}}

\def\qed{\unskip\nobreak\hfil
\penalty50\hskip1em\null\nobreak\hfil{$\square$}
\parfillskip=0pt\finalhyphendemerits=0\endgraf}
\begin{document}
%
\title{
``Model and Run" Constraint Networks\\ with a MILP Engine
}
\author{
Thierry Petit\\
}
\institute{
Worcester Polytechnic Institute\\
tpetit@\{wpi.edu, mines-nantes.fr\}
}
\maketitle
\begin{abstract}
Constraint Programming (CP) users need significant expertise 
in order to model their problems appropriately, notably to select propagators and search strategies. This puts the brakes on a broader uptake of CP. 
In this paper, we introduce \MICE, a complete Java CP modeler that can use any Mixed Integer Linear Programming (MILP) solver as a solution technique. 
Our aim is to provide an alternative tool for democratizing the ``CP-style" modeling thanks to its simplicity of use, with reasonable solving capabilities. 
Our contributions include new decompositions of (reified) constraints and constraints on numerical variables. 
\end{abstract}

\section{Introduction}
\begin{sloppypar}
Modern MILP solvers, e.g., \CPLEX~or \GUROBI,~can be defined as 
``model and run" systems~\cite{DBLP:conf/cp/Puget04}. They embed dynamic procedures so as to to make simple models  solved 
efficiently~\cite{artigues12}. Conversely, despite CP is a declarative technology, CP users need significant expertise 
in order to model their problems appropriately. This expertise is required to select propagators and to state a search strategy, which are main components of the modeling phase. 
This may put the brakes on a broader uptake of CP. 

In this paper, we present \MICE, a simple API to solve CP models, which may use any MILP engine called from Java.  
Our objective is \emph{not} to design a winner between classical, propagation-based CP solvers and MILP engines. 
Obviously, some problems are best solved by CP and other 
best solved using MILP; and other require hybridization. Our aim is to help democratizing the ``CP-style" modeling standards by providing a black-box tool where models are 
just variables plus constraints, with solving capabilities good enough for prototyping and teaching. 
This first step might convince practitioners who are not researchers in optimization or just discover the existence of CP to further increase their expertise. 
Without such initiation, some developers may not spend time in learning complex concepts, even less a language specifically dedicated to high-level CP-style modeling. 
\MICE~is thus not a competitor to such languages anymore. This idea motivates the design of an API in a popular programming language, where a model just consists of method calls for stating variables and constraints. 

As CP toolkits are not restricted to a list of linearized global constraints, we provide new theoretical contributions regarding non linear constraints with arithmetic operators (e.g., $z=x_i \times x_j$ or $z=x_i^k$), negative and positive  
tables, and reification. 
We propose extensions of the CP constraints \DEVIATION~and \INTERDISTANCE~to numerical variables. Such case studies show the benefits of the approach for 
tackling problems that usually require to hybridize discrete and continuous solvers, providing that real constraints can be linearized. Efficiency of solving may then be 
very good when \MICE~is associated with a commercial MILP solver. A result of our research is that 
all new decompositions useful to design \MICE~revealed to be quite simple. 
\end{sloppypar}
\section{Tables, Arithmetic Constraints and Logical Operators}
MILP models support linear constraints, objectives and linearized logical constraints. CP models 
make no restriction about the constraints. 
On the other hand, CP variables must range over a finite domain of values. This domain constraint can be encoded by linear constraints, 
which permitted to linearize many CP global constraints~\cite{wal1997,DBLP:conf/cp/Refalo00,belall16}. 
\begin{lemma}[Domain linearization (from Refalo's statement~\cite{DBLP:conf/cp/Refalo00})]\label{lem:domain}~\\
Let $D = \{v_1, \ldots, v_m \}$ be the domain union of a model. All the domains $\{d_1,\ldots,d_n\}$ of variables in $X$ can be stated using a set $B$ of $O(\sum_{i \in \{1,\ldots,n\}} |d_i|)$ binary variables and $O(n)$ inequalities. 
$\forall v_j \in D$, $\forall x_i \in X$, if $v_j \in d_i$ state a binary variable $b_{ij}$. $\forall x_i$, state:   
$(\sum_{v_j \in d_i} v_jb_{ij}) - x_i = 0~~~\textrm{and}~~\sum_{b_{ij}: v_j \in d_i} b_{ij}= 1.$
\end{lemma}
Domains and table constraints can be represented by linear constraints. As tables are generic and domain are finite, there is no \emph{theoretical} dominance property between MILP and CP concerning modeling capabilities.
CP models are generally more concise than linear models with (global) constraint decompositions. However, some 
problems require table constraints (without semantics), e.g., in computational biology~\cite{DBLP:conf/ijcai/PetitP16}. 
Some constraint checkers may not be tractable, even for a ground solution (for instance, 
min. overlapping variable given a set of rectangles in a fixed plane, stated by length and width variables). 
At last, it is probably easier to efficiently extend a model with numerical variables using MILP, without the need to discretize space. 
%
Modeling in CP and MILP is therefore more complex to compare than it at first seems,  
which also justifies to investigate how a MILP solver performs as a solution technique for a CP modeler. 

A CP model is not restricted to a set of global constraints whose linear formulations are found in the literature. 
We need (non linear) constraints built from arithmetic operators, negative and positive tables, and efficient reification. 
We investigate appropriate solutions to these issues. Especially, we show that table constraints provide a good way to deal with 
arithmetic operators and can compactly be reified using linear constraints.
 
A decomposition of positive table constraints is introduced in~\cite{belall16}: add one binary variable $b_{\tau_k}$ per tuple  $\tau_k$, 
and state that each variable $x_i$ should be constrained by $x_i = \sum_{k=1}^{|T|} b_{\tau_k}.\tau_k[x_i]$. 
This formulation needs two additional constraints: (1) $\sum_{k=1}^{|T|} b_{\tau_k}\leq 1$. If not, a constraint on $\{x_1,x_2,x_3\}$ 
with $T = \{(1,1,1), (3,3,3)\}$ would be considered as satisfied if $x_1=x_2=x_3=4$. (2) $\sum_{k=1}^{|T|} b_{\tau_k}\geq 1$, to have at least one tuple satisfied. 
As the domain constraint (Lemma~\ref{lem:domain}) ensures unicity of values assigned to variables through some binary variables, 
we can suppress the first one using a different decomposition. Our new linearization is exclusively stated on 
binary variables and can be naturally modified to encode negative tables and/or compact reified tables.
\begin{lemma}[Positive table linearization]\label{lem:pos} Let $c$ be a constraint defined on $var(c) \subseteq X$ by a set 
$T = \{\tau_1, \ldots, \tau_{|T|}\}$ of positive tuples. We introduce a set of binary variables $B_T$, one-to-one mapped with 
tuples in $T$. $c$ can be represented by the domain decomposition 
of $X$ and $O(|T|)$ inequalities and binary variables. For each tuple $\tau_k \in T$ and its corresponding variable $b_{\tau_k} \in B_T$ state:
$|var(c)|b_{\tau_k} - (\sum_{i, v_j \in d_i: v_j = \tau_k[x_i]} b_{ij}) \leq 0.~(1)$
In addition, state once the constraint:
$\sum_{k=1}^{|T|} b_{\tau_k} \geq 1.$
\end{lemma}
\begin{proof}
For each tuple, constraint \emph{(1)} ensures that if $b_{\tau_k}=1$ then 
any variable $x_i \in var(c)$ takes the value $\tau_k[x_i]$. 
From such tuple constraints and Lemma~\ref{lem:domain},  $\sum_{k \in \{a,\ldots,|T|\}} b_{\tau_k}$ is at most equal to $1$, otherwise two tuples in $T$ would be equal.   
The last inequality constrains $\sum_{k \in \{a,\ldots,|T|\}} b_{\tau_k}$ to be at least equal to $1$, i.e., the assignment of $var(c)$ corresponds to one tuple of $T$. 
\qed
\end{proof} 

\begin{lemma}[Negative table linearization]\label{lem:neg} Let $c$ be a constraint defined on $var(c) \subseteq X$ by a set 
$T = \{\tau_1, \ldots, \tau_{|T|}\}$ of forbidden tuples. $c$ can be represented by the domain decomposition 
of $X$ and $O(|T|)$ inequalities. $\forall \tau_k \in T$ state:
$ (\sum_{i, v_j \in d_i: v_j = \tau_k[x_i]} b_{ij})  \leq |var(c)| - 1.$
\end{lemma}
\begin{proof}
For $c$ to be satisfied, $\forall\tau_k \in T$ at least one variable $x_i \in var(c)$ should take a value $v_j \neq \tau_k[x_i]$. 
From Lemma~\ref{lem:domain}, if $x_i = v_j$ then $b_{ij} = 1$, otherwise $b_{ij} = 0$. 
$\sum_{i, v_j \in d_i: v_j = \tau_k[x_i]} b_{ij}$ must be $<|var(c)|$. \qed
\end{proof} 

Some constraints can be  
concisely and thus quite efficiently represented by tables, e.g., arithmetic operators. For instance, constraint $z = x_i^k$ where $z$ and $x_i$ are integer domain variables and 
$k$ a positive integer can be stated using $|d_i|$ allowed tuples. 
The constraint $z = x_i\times x_j$ requires $|d_i| \times |d_j|$ allowed tuples to be stated. 
Many other constraints are also appropriate for table-based linear representation.  
The \ELEMENT~constraint can be represented using Lemma~\ref{lem:pos}, thanks to its positive tuples. 
This formulation is distinct from the existing ones~\cite{DBLP:conf/aaai/HookerOTK99,DBLP:conf/cp/Refalo00} and does not add integer variables for indexes. 
\begin{lemma}[\ELEMENT~linearization]\label{lem:element}
Let $x_i$ and $x_j$ be two integer domain variables and $t$ an array of integers, not necessarily all distinct. 
\emph{\ELEMENT}~is statisfied if an only if  $x_j$ is equal to the $x_i^{th}$ value in the array $t$. 
From Lemma~\ref{lem:pos}, \emph{\ELEMENT} on ($x_i,x_j$) can be decomposed using $O(\min(|t|,|d_i|,|d_j|))$ inequalities and binary variables.
\end{lemma}
\begin{proof} 
$\forall t[k] \in t$, add the tuple $(k,t[k])$ in the table if $k \in D(x_i)$ and $t[k] \in D(x_j)$. 
\qed
\end{proof}

For sake of space, we do not provide a broader list of constraints that would likely be represented by tables. 
In some cases, however, the number of tuples may be prohibitive. Let us consider one noticeable case against tables: \ALLDIFF, 
that holds iff $\forall (x_i,x_j) \in var(c), x_i \neq x_j$. 
A concise ad-hoc linearization exists~\cite{wal1997,DBLP:conf/cp/Refalo00}.  
In contrast, a positive table would have a number of tuples of the order of magnitude of the number of permutations of the values 
 in domains. 
The case of a negative table is worst. 
As some constraints may be difficult to linearize in an ad hoc fashion, we consider two classical CP notions that may help for using tables: 
decompositions and implied constraints. 
\paragraph{Ad Hoc CP Decompositions.}
In the CP literature, many papers put the focus on decomposing constraints not only using linear equations, 
but using other (simpler) constraints standardly provided by CP solvers (see, e.g.,~\cite{DBLP:conf/ijcai/BessiereKNQW09}). In MILP 
there is no obvious link between the use of \emph{global constraints}~\cite{DBLP:conf/cp/BessiereH03} and solving process efficiency, as such constraints must be expressed through linear equations. 
In the case of \ALLDIFF, a well-known decomposition is to state for each variable pair a binary inequality. 
Although representing a binary inequality through table constraints is, to say the least, convoluted, we may note that decomposing makes possible a table-based formulation (using $O(d)$ negative tuples per inequality). 
Concerning systematic decomposition schemes, the generic automaton based decomposition of the global constraint catalog~\cite{DBLP:journals/constraints/BeldiceanuCDP07} is well suited to table constraint  based 
linearization, as transition and signature constraints often correspond to positive tables with small scopes. 
It can be applied on those constraints that can be checked by scanning once through their variables. 
For instance, using Lemmas~\ref{lem:pos} and~\ref{lem:neg}, Example 2 of section 3 in~\cite{DBLP:journals/constraints/BeldiceanuCDP05} provides 
an exploitable linear decomposition for lexicographical constraints. 
\paragraph{Implied Constraints.} 
The \GCC~\cite{DBLP:conf/aaai/Regin96} would likely be implied to any other one, as it refers to occurrences of values in solutions. 
The \GCC~has some particular properties that make it easy to learn from a model~\cite{DBLP:conf/ijcai/BessiereCP07}. Therefore, it can be an efficient tool 
for reducing tuples of a negative table (all tuples violating the \GCC~are not added to the table). 
%
%
\begin{definition}[\GCC]\label{def:gcc}
Let $X=\{x_1,\ldots,x_n\}$ be a set of variables, $t$ an array of values. Each $t[k] \in t$ is associated with two 
integers $\underline{t[k]}$ and $\overline{t[k]}$, $0 \leq \underline{t[k]} \leq \overline{t[k]}$. 
\emph{\GCC} holds if and only if for all indexes $k$: 
$\underline{t[k]}\leq |\{ i: x_i = t[k] \}| \leq \overline{t[k]}.$
\end{definition}
The next Lemma provides a linear formulation of \GCC~that adds $O(|t|)$ binary inequalities to the model. It differs from the existing one~\cite{DBLP:conf/cp/Refalo00} because it is related to a simpler version of the \GCC. In 
Definition~\ref{def:gcc}, bounds on value occurrences are stated using integer values, not variables. In our context variables are useless. 
\begin{lemma}[\GCC~linearization]\label{lem:gcc}
We use the notations of 
Lemma~\ref{lem:domain} and Definition~\ref{def:gcc}. 
For all $t[k] \in t, B_{t[k]} = \{ b_{ij} \in B: v_j = t[k]\}$.  
\emph{\GCC}~can be represented by the domain decomposition 
of $X$ and $O(|t|)$ inequalities. $\forall t[k] \in t$ state:
$(\sum_{b \in B_{t[k]}} b) \geq \underline{t[k]} \textrm{~and~} (\sum_{b \in B_{t[k]}} b) \leq \overline{t[k]}.$
\end{lemma}
\begin{proof}
From Lemma~\ref{lem:domain}, $\forall b_{ij} \in B, b_{ij} = 1$ if and only if $x_i = v_j$. 
\qed
\end{proof}
Using a \GCC~as an implied constraint 
introduces only $O(|t|)$ inequalities and no additional variable. 
Let's come back to the \ALLDIFF~example.  It is worthwhile to notice that adding the implied \GCC~eliminates all the tuples of an \ALLDIFF~decomposed by a negative table. 
We obtain the ad hoc \ALLDIFF~Refalo's linearization from Lemma~\ref{lem:gcc}. This is an extreme case but, more generally, 
in practical problems implied \GCC's~can have shrink bounds~\cite{DBLP:conf/ijcai/BessiereCP07}.
\paragraph{Reification.}
Instead of merely posting a constraint $c$ it is often useful to reflect its truth value into a binary variable $r_c$. This process is called \emph{reification}. It permits to express 
logical constraints such as $c_i \vee c_j$, or $\neg c$. Table constraints reification relies to ideas of Koster's partial constraint satisfaction formulation in the context 
of frequency assignment problems~\cite{koster}.  
Basically, we can extend the table with variable $r_c$. 
Naively one may state that $r_c$ is equal to $1$ for allowed tuples \emph{and} $0$ for forbidden ones. 
The number of tuples would be all the combinations of values of the cartesian product of $|var(c)|$. 
We propose a decomposition that does not increase the number of tuples of $c$. 
\begin{lemma}[Table constraint reification]\label{lem:reif} We keep the notations of Lemma~\ref{lem:pos}. The set $T$ can either represent allowed or forbidden tuples. 
The reification of a table constraint $c$ can be represented by the domain decomposition 
of $X$, the binary variable $r_c$ used to express the truth value of $c$ and $O(|T|)$ inequalities and binary variables.
\emph{(1)} $\forall \tau_k \in T$ state:
$|var(c)|b_{\tau_k} -  (\sum_{i, v_j \in d_i: v_j = \tau_k[x_i]} b_{ij})  \leq 0$, and:
$(\sum_{i, v_j \in d_i: v_j = \tau_k[x_i]} b_{ij})  - |var(c)|b_{\tau_k}  \leq |var(c)|-1.$
\emph{(2)} In addition, if tuples in $T$ are allowed tuples state:
$(\sum_{k \in \{a,\ldots,|T|\}} b_{\tau_k}) - r_c = 0,$
otherwise state:
$(\sum_{k \in \{a,\ldots,|T|\}} b_{\tau_k}) + r_c = 1.$
\end{lemma}
\begin{proof} $\forall \tau_k \in T$ the first inequality ensures that 
if $b_{\tau_k}=1$ then all $x_i \in var(c)$ take the value $\tau_k[x_i]$. If $b_{\tau_k}=0$ it is always satisfied. 
The second one ensures that if $b_{\tau_k}=0$ not all the variables  $x_i \in var(c)$ take the value $\tau_k[x_i]$. 
If $b_{\tau_k}=1$ it is always satisfied. From Lemma~\ref{lem:domain},  $\sum_{k \in \{a,\ldots,|T|\}} b_{\tau_k} \leq 1$.  From \emph{(2)}, 
$r_c = 1$ if and only if the variables in $var(c)$ are fixed with a tuple of $T$ (positive table) or not in $T$ (negative table).
\qed
\end{proof} 
If the \GCC~is added to the model as an implied constraint related to a reified table, we also need to reify it. For this purpose, or simply as this may be useful in models, we introduce a new linear decomposition.  
\begin{lemma}[\GCC~reification]\label{lem:reigcc}
We use Lemma~\ref{lem:gcc} notations. Let $r_c$ be the binary variable for the truth value of the \emph{$\GCC$} to be reified. 
We create $O(|t|)$ binary variables $r^{k+}_c$, mapped with values in $t$ and $O(|t|)$ binary variables $r^{k-}_c$, also mapped. 
$R$ is the set of  all $r^{k-}_c$ and all $r^{k+}_c$ variables, of size $2|t|$. 
The reifed \emph{$\GCC$} is obtained by the domain decomposition of $X$. In addition, $\forall t[k] \in t$ state:
{
$$(\sum_{b \in B_{t[k]}} b) - \underline{t[k]}r_c^{k-} \geq 0.~~~~~~~~~~~~~~~~(1) 
~~~~~~~~~~~~~~~~(\sum_{b \in B_{t[k]}} b) - r_c^{k-}(n + 1) \leq \underline{t[k]} -1.~~~~~~~(2)$$
$$(\sum_{b \in B_{t[k]}} b) + (\overline{t[k]}+1)r_c^{k+} \geq \overline{t[k]}+1~~~~~~~~~~~~~~~~~~~~~~~~~~
(\sum_{b \in B_{t[k]}} b) + nr_c^{k+} \leq \overline{t[k]}+n.$$
}
Moreover, state once:~~
$(\sum_{r\in R} r) -r_c \leq 2|t| -1.~~~~(3)$~~~~~~~~~$(\sum_{r\in R} r) - 2|t|r_c \geq 0.~~~~~(4)$
\end{lemma}
\begin{proof}
We first consider the case of one value $t[k] \in t$. Without loss of generality we restrict to $\underline{t[k]}$ (the case of maximum occurence is symmetrical). 
 If $r^{-k}_c = 0$ then \emph{(1)} is always satisfied and \emph{(2)} is satisfied if and only if $\sum_{b \in B_{t[k]}} \leq \underline{t[k]} -1$, i.e., the 
minimum required number of occurrences of $t[k]$ is not reached. If $r^{k-}_c = 1$ then \emph{(2)} is always satisfied and \emph{(1)} is satisfied 
if and only if $(\sum_{b \in B_{t[k]}} b) \geq \underline{t[k]}$. 
Then, consider all values. If $\sum_{r\in R} r  = 2|t|$ then all lower and upper cardinalities are in the request bounds, 
$r_c = 1$ to satisfy constraint \emph{(3)}. Otherwise, \emph{(3)} is always satisfied. 
If $\sum_{r\in R} r < 2|t|$ then $r_c = 0$ to satisfy constraint \emph{(4)}. Otherwise, \emph{(4)} is always satisfied. 
The Lemma holds. 
\qed
\end{proof}
For sake of space we do not provide a broader list of reified global constraints. It is worth saying that the general idea behind reification is generic: isolate a satisfaction property and then 
use the ``Big-M" principle to obtain the boolean value. This general principle can also be used for the constraints on numerical variables presented in next section. 
We claim that its simplicity is a valuable result. 
\section{The \MICE~modeler}
\MICE~is a Java modeler devoted to solve CP models using any MILP engine that can be called from Java. 
\MICE~embeds predefined global constraints\footnote{A set of global constraints similar to existing CP solvers such as Choco or OR-Tools.}, tables, arithmetic and logical operators. 
The two central ideas of \MICE~design are the following.\\
1. \emph{Simplicity of use.} \MICE~is primarily designed to provide users who are not expert in optimization 
with an autonomous tool for solving CP models. A \MICE~model is just defined by stating variables and high-level constraints, set on integer or numerical variables,  
to tackle also discrete-continuous problems.\\
%
2. \emph{Modular design.} \MICE~provides its own API for stating linear constraints. Therefore, all predefined constraints as well as user decompositions are stated 
using \MICE~objects. To plug \MICE~to a new MILP solver, the main class, called \texttt{Solver}, should be augmented by a call to the method creating a new model in the MILP solver. In addition, 
a unique new class must be created. This class implements an interface that states the methods used to create mathematical variables and linear constraints within the MILP engine (this makes the link 
with \MICE~linear expressions), to call specific methods for running the MILP solver, limiting time, and some getters (value of a mathematical variable, solver statistics, etc.). 
\MICE~can solve satisfaction and optimization problems. 
Logical constraints are set through tables on reification variables. Given any two reified constraints $c_1$ and $c_2$, \texttt{ReifOR} is a table\footnote{To simplify the description we do not use Lemma~\ref{lem:reif}, although the two formulations can obviously be considered.} on the truth variables $r_1$ and $r_2$ 
and a new variable $r$: the allowed tuples on $\{r_1,r_2,r\}$ are $\{(0,0,0),(0,1,1),(1,0,1),(1,1,1)\}$. Stating \texttt{OR} is then $r=1$. The cases of \texttt{AND}, $\neg$ and $\rightarrow$ (implication)   
are similar. There is no limitation about the number of logical operator combinations on reified constraints.


\MICE~permits to use in the same model global constraints on integer and numerical variables. 
Consider that every variable $x$ comes up with a lower bound $\underline{x}$ 
and an upper bound $\overline{x}$ (without loss of generality, these values can be the minimum and maximum value encodable by the computer). 
Next definition (derived from the Santa Claus benchmark~\cite{DBLP:journals/corr/FagesCP14}) generalizes the \DEVIATION~constraint~\cite{DBLP:conf/cpaior/SchausDDR07}, stated to obtain balanced solutions to 
combinatorial problems, by considering a numerical variable for the mean instead of an (arbitrary) integer value.
\begin{definition}[Extended \DEVIATION]~\label{def:deviation}
Let $X=\{x_1,\ldots,x_n\}$ be a set of variables (integer or numerical).
Let $z$ and $s$ be {\bf numerical variables}. 
\emph{\DEVIATION}~holds if and only if: $nz = (\sum_{i=1}^n x_i)\textrm{~~~and~~~}s = (\sum_{i=1}^n |x_i - z|).$
\end{definition}
To linearize \DEVIATION~we first need to linearize the constraint $\mathit{abs} = |x-y|$, where $x,y$ and $\mathit{abs}$ are mathematical variables of any type. 
\begin{lemma}[$\mathit{abs} = |x-y|$~linearization]~\label{lem:dist}
We use the following notations: Given $\underline{x},\overline{x},\underline{y}$, and $\overline{y}$, $a$ is the minimum possible value of $|x-y|$ and $A$ the maximum possible value. 
$d$ is the minimum possible value of $x-y$ and $D$ the maximum possible value. 
We then distinguish three cases. 
\begin{enumerate}
\item $D\leq 0$. Add $\mathit{abs}   - y + x = 0$.
\item $d \geq 0$. Add $\mathit{abs}   - x + y = 0$.
\item $d < 0 < D$. Add $a \leq \mathit{abs} $ and $\mathit{abs}  \leq A$, define three variables $\mathit{dif}, \mathit{difp}$ and $\mathit{difn}$, a binary variable $b$ and state: 
{\small
$$\mathit{dif} - x + y = 0.~~~~~~~~~~~~~~~~~~~(1)~~~~~~~~~~0 \leq \mathit{difp}~~\textrm{and}~~\mathit{difp}\leq D.~~~~~~~~~~~~~~(2)$$
$$0 \leq \mathit{difn}~~\textrm{and}~~\mathit{difn}\leq |d|.~~~~~(3)~~~~~~~~~~d \leq \mathit{dif}~~\textrm{and}~~\mathit{dif}\leq D.~~~~~~~~~~~~~~~~~(4)$$
$$\mathit{dif} - \mathit{difp} + \mathit{difn} = 0.~~~~~~~~~~(5)~~~~~~~~~~\mathit{difp} - Db \leq 0.~~~~~~~~~~~~~~~~~~~~~~~~~~~~~~(6)$$
$$\mathit{difn}  + |d|b \leq |d|.~~~~~~~~~~~~~~~~~~~(7)~~~~~~~~~~\mathit{abs}  - \mathit{difp} - \mathit{difn}=0.~~~~~~~~~~~~~~~~~(8)$$
}
\end{enumerate}
\end{lemma}
\begin{proof} Cases 1 and 2 are obvious. 
If $d < 0 < D$, constraints (2) (3) and (4) state the bounds of $\mathit{difp}$, $\mathit{difn}$ and $\mathit{dif}$. 
From (1),  $\mathit{dif} = x - y$. 
(5) states $x-y = \mathit{difp} - \mathit{difn}$. 
From (2) and (3) $\mathit{difp}\geq 0$ and $\mathit{difn}\geq 0$. We have: 
either $\mathit{difp}> 0$, then from (6) $b=1$, and from (7) $\mathit{difn}=0$: from (8), $\mathit{abs} =x-y$; 
or $\mathit{difn}> 0$, from (7) $b=0$, from (6) $\mathit{difp}=0$: from (8), $\mathit{abs} =y-x$;
otherwise, $\mathit{abs} =\mathit{difp}=\mathit{difn}=0$. 
\qed
\end{proof}
From Lemma~\ref{lem:dist} and Definition~\ref{def:deviation}, we decompose~\DEVIATION. 
%
%
\begin{lemma}[\DEVIATION~linearization]~\label{lem:deviation}
We denote by \emph{\DISTANCE}($x,y,\mathit{abs}$) the linearization of constraint $\mathit{abs} = |x-y|$ of Lemma~\ref{lem:dist}. Let $Z = \{\mathit{abs}_1, \ldots, \mathit{abs}_n\}$ be 
numerical variables. 
\emph{\DEVIATION}~(Definition~\ref{def:deviation}) can be represented by the following set of linear constraints: 
$\forall x_i \in X,  \textrm{\em\DISTANCE}(x_i,z,\mathit{abs}_i).$
In addition state once:  
$nz - (\sum_{i=1}^n x_i) = 0\textrm{~~~and~~~}s - (\sum_{i=1}^n \mathit{abs}_i) = 0.$
\end{lemma}
The \INTERDISTANCE~constraint~\cite{reg97,DBLP:conf/aaai/QuimperLP06} holds on integer variables $\{x_1,\ldots,x_n\}$ and a constant $p$ 
if and only if $|x_i-x_j|\geq p$ for all $i \neq j$. Our decomposition extends it to variables of any type and a numerical variable for $p$: 
$\forall x_i, x_j, i < j$, state   $\textrm{\DISTANCE}(x_i,x_j,\mathit{abs}_{ij})$ and $\mathit{abs}_{ij}\geq p$. 
\section{Experiments and Conclusion}
\begin{sloppypar}
In our experiments \MICE~is coupled with \GUROBI~\cite{gurobi} on a I7 with 8GB of RAM. 
We consider CP solvers that can be called from Java, and, to truly compare with classical CP, with an API for stating advanced search strategies. 
%
In this paper our goal is to compare with CP, not with hybrid techniques or systems adding implicit constraints to the model, even less to solvers specific to one particular problem. 
We consider solver hybridization in the discrete-continuous case, because discrete-continuous models can directly be stated using \MICE~providing that real constraints can be linearized. 

Table~1~focuses on satisfaction problems stem from Choco 3.3.3~\cite{choco3} and OR-Tools~\cite{van2014or} sample directories (see CSPlib~\footnote{http://www.csplib.org/} for 
descriptions) that use the \ALLDIFF~global constraint 
(we report the best result among the two solvers/different propagator options). 
We compare linear decompositions of \ALLDIFF~through negative binary tables and Refalo's formulation. Surprisingly, the results are quite similar using tables or not on this set of instances. 
\begin{table}[!t]\label{tab:sat}
\parbox{.48\linewidth}{
\begin{center}
\small
\scalebox{0.7}{
\begin{tabular}{|c|c|c|c|c|}
\hline
Problem & Ad hoc  & Bin. tables & CP & CP \\
             &\ALLDIFF &   (Lemma 3)  & default & best found \\
             &                  &                     & strategy & strategy \\ 
            \hline
   	      Sudoku 1 to 6 & { 0 sec.} & { 0 sec.} & 0 sec. & 0 sec.  \\
	        prob03-10 &  0.1 sec. & 0.1 sec.  & 0 sec. &  0 sec. \\
	     prob03-20 &  1.5 sec. & 2.1 sec.  &  0.2 sec. &  0.2 sec. \\
	     prob07-12 & 0.1 sec. &  0.1 sec. & 0.2 sec. & 0 sec. \\
	          prob07-14 & 0.3 sec. &  0.3 sec. & 0.7 sec. & 0.1 sec. \\
	        prob07-16 & 0.4 sec. &  0.4 sec. &  20 sec. &  0.1 sec.  \\
	        prob07-18 &  0.6 sec. &  0.6 sec. & $>$ 1 min. & 0.1 sec. \\
	        prob019-5 & {0 sec.} & 0.1 sec. & 0 sec. & 0 sec. \\
 	     prob019-6 & 1 sec. & 9.8 sec. & 1.7 sec. & 1.7 sec. \\
	     prob019-7 & 3.5 sec. & { 3.1 sec.}  & $>$ 1 min. &  6.4 sec. \\

 \hline
\end{tabular}
}
\end{center}
\caption{\MICE~with \GUROBI: satisfaction problems. The CP best strategies include 
specific static orders and more advanced techniques, such as impact based search with 
parameter values specific to a given instance (prob19-7). }
}
\hfill
\parbox{.47\linewidth}{
{
\begin{center}
\small
\scalebox{0.7}{
\begin{tabular}{ |p{1.1cm}|p{0.9cm}|p{1.7cm}|p{1.1cm}|p{0.9cm}|p{0.9cm}|p{0.9cm}| }
\hline
Series & \MICE & \MICE~av. & \MICE~av. & CP & CP av. & CP av. \\
(10 instances)                &   optimal proofs & obj. value & time (sec.) &  optimal proofs & obj. value & time (sec.) \\
 \hline
  	   bqp50 & 100\% & 1926.8 & 60.1 &  60\% & 1903.9 & 447.3  \\
	   \hline
g05\_20 & 100\% & 64.9 & 1.9  & 100 \% &  64.9 & 3.6 \\

\hline
g05\_30  & 100\% & 138.8  & 40.1 & 0 \% &  137.1 & 600 \\
\hline
g05\_40  & 40\% & 244.9 (3.2\%) & 549.1 & 0 \% & 228.2 & 600  \\
\hline
\end{tabular}
}
\end{center}
\caption{\MICE~with \GUROBI~vs CP: maximization problems.}
}
{
\begin{center}
\small
\scalebox{0.7}{
\begin{tabular}{|c|c|c|c|c|}
\hline
nb. kids & max. price & nb. gifts & \MICE  & Choco-Ibex \\
\hline
3 & 25 & 5 & 0 sec. & 0 sec. \\
6 & 50 & 10 & 0 sec. & 0.8 sec. \\
9 & 75 & 15 & 0.3 sec. & 371.6 sec. \\
12 & 100 & 20 & 32.3 sec. & $>$ 600 sec. \\ 
15 & 125 & 25 & 341.5 sec. & $>$ 600 sec. \\ 
\hline
\end{tabular}
}
\end{center}
\caption{\MICE~with \GUROBI~vs Choco-Ibex.}
}
}
\vspace{-0.7cm}
\end{table}
We do not consider, in a short paper, reproducing known results on optimization problems with ad hoc global constraint formulations~\cite{wal1997,DBLP:conf/cp/Refalo00,belall16} that are also implemented in \MICE. 
Using \MICE~in this context would lead to the same conclusions. 
Rather, Table~2~reports results about the CP-style modeling features of \MICE~to solve the Max-Cut non-linear (quadratic) optimization problem that occurs in 
physics applications~\cite{liers04}, without any handcrafted model transformation. 
Given an undirected graph with weighted edges $G_w=(V,E_w)$, Max-Cut is the problem of finding a cut in G of maximum weight. 
A variable $x_i$ is stated for each vertex in $V$. $x_i = 1$ if vertex $v_i$ is in S and $x_i=-1$ otherwise. We maximize $\frac{1}{2}\sum_{i<j} w_{ij}(1 - x_ix_j)$. 
We encoded this problem in Choco 3.3.3 and \MICE.
In \MICE, the default product constraint corresponds to a positive table. 
In CP, the best strategy we found is {\em DomOverWdeg}~\cite{DBLP:conf/ecai/BoussemartHLS04} 
and first assign vertices.
We used Beasley instances from the OR-Libary and problems from g05\_60 Rudy instances\footnote{See http://biqmac.uni-klu.ac.at/biqmaclib.html}, with a 10 min. time-limit. 
Graphs have respectively 50 nodes and weights in $[-100,100]$, and 30-50 nodes unweighted. 
At last, we evaluated our modeler on the Santa Claus benchmark, recently solved by hybridizing discrete and continuous CP solvers, namely Choco 3.3.3 and Ibex 2.3.1~\cite{ibex}. 
We 
use the model provided in~\cite{DBLP:journals/corr/FagesCP14}, with  
\ALLDIFF~and \ELEMENT~on integer variables and \DEVIATION~with a continuous variable for the mean. 
Table~3~shows time for proving the optimal solution, with a 10 minutes limit. 

Throughout the development of \MICE, we revisited linear formulations of global and table constraints and their reification, 
non linear constraints built from arithmetic/logical operators, 
and extensions of \DEVIATION~and \INTERDISTANCE~to numerical variables. 
New integer constraint decompositions exclusively 
involve new binary variables. An interesting result is that all linear formulations in \MICE~are simple. 
We do not claim that any problem can be tackled using a  black-box MILP-based CP modeler 
(scalability issues may even occur with the best ad hoc MILP models). Our results are good in the context of promoting the CP-style modeling with a simple library, implemented in  
a popular programming language. 
Future work includes implementing most recent advances in CP models linearization, e.g., domain refinement~\cite{belall16}, 
and link \MICE~with model acquisition systems~\cite{DBLP:conf/cp/BeldiceanuS12,DBLP:conf/ijcai/ArcangioliBL16}.  
\end{sloppypar}
\bibliographystyle{splncs03}
\bibliography{tp17}

\end{document}